\documentclass[twocolumn,10pt]{article}
\usepackage[letterpaper,textwidth=6.75in,textheight=9.25in,top=1in,left=0.875in,columnsep=0.25in,footskip=0.45in]{geometry}

\usepackage[round]{natbib}

\usepackage[T1]{fontenc}
\usepackage{amsmath,amssymb,amsthm}
\usepackage{booktabs}
\usepackage{graphicx}
\usepackage{float}
\usepackage{xcolor}
\usepackage[hidelinks]{hyperref}
\usepackage{url}
\usepackage{titlesec}
\titleformat*{\section}{\large\bfseries\raggedright}
\titleformat*{\subsection}{\normalsize\bfseries\raggedright}
\titlespacing*{\section}{0pt}{2.0ex plus 0.5ex minus .2ex}{1.5ex plus 0.3ex minus .2ex}
\titlespacing*{\subsection}{0pt}{1.8ex plus 0.5ex minus .2ex}{0.8ex plus .2ex}
\newtheorem{proposition}{Proposition}
\newtheorem{lemma}{Lemma}
\theoremstyle{remark}

\newcommand{\Rsel}{R}
\newcommand{\lamhat}{\hat\lambda}
\newcommand{\Ucp}{U_{\mathrm{CP}}}
\newcommand{\pizero}{\pi_0}
\newcommand{\pfire}{P_{\mathrm{fire}}}

\begin{document}

\twocolumn[
\begin{center}
{\Large\bfseries Certified by Abstention: Distribution-Free Guarantees for\\ Chain-of-Thought Verifiers at Small Calibration Budgets\par}
\vspace{1.4em}
{\large Arjun Balaji\par}
\vspace{0.3em}
{Columbia University\par}
{\texttt{arjun.balaji.02@gmail.com}\par}
\vspace{1.8em}
\end{center}
]

\begin{abstract}
Signals that predict whether a chain-of-thought (CoT) trace is correct are compared by AUC, but deploying one requires a threshold with a guarantee. We ask what distribution-free selective guarantees deliver for CoT verifiers at realistic calibration budgets of tens to a few hundred labelled problems, using seven open models, five verifier signals and 37{,}000 graded traces. The central observation is \emph{validity by abstention}: an $(\alpha,\delta)$-valid procedure that issues a certificate with probability $\pfire$ bounds the failure probability of an issued certificate only by $\delta/\pfire$, so a certificate that rarely fires can be valid and wrong every time it is used. In a simulation with known risk the standard certificate fails in at most $0.3\%$ of calibration draws but in up to $69\%$ of those in which it fires. A \emph{certification floor} and a \emph{lattice condition} for Benjamini--Hochberg conformal selection explain why certificates abstain at these budgets, and the data bear them out: the standard certificate returns nothing or a large accepted set, and an unreadable residual-stream probe buys two to three times the coverage of the readable signals, an edge a cross-fitted reconstruction cannot recover linearly from the readable features. We then give a \emph{floor-started fixed-sequence certificate}, valid without monotonicity assumptions, that covers more than the Bonferroni certificate on every model--signal pair and raises coverage at the non-vacuous target $0.75\pizero$ from $0.05$ to $0.16$, although the floor keeps absolute coverage small. Finally, a certificate cannot see what matters after deployment: under benchmark shift the error among accepted traces tracks the new task's base error, and under best-of-$n$ selection against the verifier it rises past the target while the empirical failure frequency stays below $\delta$, because abstention absorbs the failures.
\end{abstract}

\section{Introduction}

Correctness of a language model's reasoning can be predicted from cheap signals: the model's token probabilities \citep{kadavath2022know}, linear probes on hidden states \citep{azaria2023internal,burns2023discovering,zhang2025reasoning}, or readouts that decode intermediate activations into vocabulary space \citep{nostalgebraist2020logit,belrose2023tuned,gurnee2026workspace}. Such \emph{verifier signals} are compared almost exclusively by ranking metrics. A deployed verifier, however, is a threshold: traces scoring above it are accepted without further checking, and the operator needs to know how often an accepted trace is wrong.

Distribution-free methods promise this. Given a labelled calibration set, selective certificates \citep{geifman2017selective}, risk-controlling procedures \citep{bates2021rcps,angelopoulos2021ltt} and conformal selection \citep{jin2023selection,gui2024conformal} return an acceptance rule whose error is bounded with finite-sample validity under exchangeability alone. Their validity is not in question. What is rarely reported is what they \emph{deliver} when labels are scarce, the target is non-vacuous, and the accepted traces are then optimized against. For CoT verification, calibration labels require graded solutions, and realistic budgets are tens to a few hundred problems. GSM8K's automatic grading makes it a convenient stand-in for the tasks that motivate this regime, where each label costs an expert's time; we therefore use 167--395 problems per model, the scale of a hand-graded set, and study larger budgets in simulation.

We study this regime with seven open-weight models (0.8B--8B parameters, three families), five verifier signals spanning the interpretability spectrum, and about 37{,}000 graded traces. The target is the \emph{selective error} $P(\text{wrong}\mid\text{accepted})$. A target $\alpha$ is \emph{non-vacuous} if it lies strictly below the model's unconditional error rate $\pizero$.

\paragraph{Contributions.}
\textbf{(1) Validity by abstention} (\S\ref{sec:theory}). An $(\alpha,\delta)$-valid procedure that issues a certificate with probability $\pfire$ bounds the failure probability of an \emph{issued} certificate only by $\delta/\pfire$ (Proposition~\ref{prop:abstain}). A certificate that rarely fires can therefore be valid and wrong every time it is used, and in a simulation with known risk the standard certificate does exactly that: it fails in at most $0.3\%$ of draws overall and in up to $69\%$ of the draws in which it fires, with a median overshoot of 4--12\% of the target (\S\ref{sec:synthetic}).
\textbf{(2) Why certificates abstain at small budgets} (\S\ref{sec:theory}, \S\ref{sec:attain}). A \emph{certification floor} (Lemma~\ref{lem:floor}) gives the smallest accepted calibration set that can certify a target, and a \emph{lattice condition} (Proposition~\ref{prop:lattice}) decides exactly when BH conformal selection can reject anything. On seven models with 67--158 calibration problems the standard certificate is bimodal, abstaining or accepting a large set; at the non-vacuous target $0.75\pizero$ it fires in at most $10\%$ of calibration draws for 15 of 21 model--signal pairs, and below the boundary only one readable signal on one model reaches $10\%$ coverage. An unreadable residual probe buys two to three times the coverage of the readable signals at the boundary target and more below it, and a cross-fitted reconstruction shows that this edge is not linearly recoverable from the 119 readable features. The exact binomial bound is the tightest non-randomized bound on its statistic, so the lever is the multiplicity correction, not the bound.
\textbf{(3) A certificate that starts at the floor} (\S\ref{sec:fs}). FS-floor is fixed-sequence Learn-then-Test \citep{angelopoulos2021ltt} with thresholds fixed from the training split, a start at the floor of Lemma~\ref{lem:floor}, and a start constant chosen by simulation before calibration. It is valid with no monotonicity assumption and covers more than the Bonferroni certificate on all 21 model--signal pairs at every target: at the non-vacuous target $0.75\pizero$ mean coverage rises from $0.05$ to $0.16$ and the certificate fires in $28\%$ rather than $10\%$ of calibration draws. The floor still binds, and a budget curve (Figure~\ref{fig:budget}) shows what larger calibration sets would buy.
\textbf{(4) What a certificate cannot see} (\S\ref{sec:shift}, \S\ref{sec:pressure}). Under benchmark shift the error among accepted traces tracks the new task's base error (Spearman $0.88$): the target is exceeded in 40 of the 45 cells where the new task is harder and in none of the 27 where it is easier; under best-of-$n$ selection against the verifier it rises with $n$ past the target for several signals while the empirical failure frequency stays below $\delta$, because abstention absorbs the failures; the selected trace's accuracy degrades.

\section{Setting}
\label{sec:setting}

\paragraph{Traces, labels, scores.} For a problem $x$ a model samples a CoT trace $z$ at temperature $0.8$; $Y=1$ if its final answer matches the gold answer. A verifier signal is a map $s(z)\in\mathbb{R}$, fit on a training split and then frozen. We accept $z$ when $s(z)\ge\lambda$. The \emph{selective risk} and \emph{coverage} of a threshold are
\begin{align*}
\Rsel(\lambda)&=P\big(Y=0\mid s(Z)\ge\lambda\big),\\
C(\lambda)&=P\big(s(Z)\ge\lambda\big),
\end{align*}
with $\Rsel(\infty)=0$. Let $\pizero=P(Y=0)$. A procedure mapping a calibration set to $\lamhat$ is \emph{$(\alpha,\delta)$-valid} if $P(\Rsel(\lamhat)\le\alpha)\ge1-\delta$. It \emph{fires} when $\lamhat<\infty$, and $\pfire=P(\lamhat<\infty)$. Our primary target is the boundary $\alpha=\pizero$; we also report $\alpha\in\{0.5\pizero,0.75\pizero\}$ and $\alpha\in\{0.05,0.10,0.20\}$. At $\alpha=\pizero$ accepting everything is itself a correct answer, so this target measures whether the certificate can confirm what is true by construction; the targets below $\pizero$ are the non-vacuous ones, and even they can be loose in absolute terms (for the 33\%-accurate model, $0.75\pizero$ is an error rate of $0.50$).

\paragraph{Signals.} Each signal is a regularized logistic regression on a fixed feature set (Appendix~\ref{app:signals}).
\emph{Confidence}: 5 summary statistics of the token log-probabilities of the trace.
\emph{J-lens}: 51 trajectory features (velocity, curvature, entropy and churn of the top-20 decoded tokens across reasoning steps) of residual-stream states decoded with the Jacobian lens of \citet{gurnee2026workspace}.
\emph{Logit lens}: the same features with the Jacobian transport disabled \citep{nostalgebraist2020logit}.
\emph{Surface}: geometric features of sentence embeddings \citep{reimers2019sentence} of the written steps.
\emph{Probe}: a linear probe on mid-depth residual vectors (mean over step positions, concatenated with the answer position).
The first four have human-inspectable features; the probe does not.

\paragraph{The certificate.} Our main procedure is a Clopper--Pearson (CP) selective threshold \citep{clopper1934use,geifman2017selective}. Sort the $n$ calibration scores decreasingly, $s_{(1)}\ge\dots\ge s_{(n)}$, and let $e_k$ count incorrect traces among the top $k$. For a grid $\mathcal K\subset\{10,\dots,n\}$ of $G\le20$ equally spaced sizes fixed in advance, test the $k-1$ traces strictly above $s_{(k)}$:
\begin{equation}
\begin{aligned}
\hat k&=\max\{k\in\mathcal K:\ \Ucp(e_{k-1},k-1;\delta/G)\le\alpha\},\\ \lamhat&=s_{(\hat k)},
\end{aligned}
\label{eq:cert}
\end{equation}
accept a test trace when $s(z)>\lamhat$, and abstain ($\lamhat=\infty$) if the set is empty. $\Ucp(e,k;\delta')$ is the one-sided $1-\delta'$ CP upper limit, the $1-\delta'$ quantile of $\mathrm{Beta}(e+1,k-e)$. We use $\delta=0.10$.

\begin{proposition}[Validity]
\label{prop:valid}
Let calibration traces be i.i.d.\ with continuous scores. Then rule \eqref{eq:cert} is $(\alpha,\delta)$-valid.
\end{proposition}

The proof (Appendix~\ref{app:proofs}) uses only that, conditional on $s_{(k)}=t$, the $k-1$ traces above $t$ are i.i.d.\ from the law of $(Y,s(Z))$ given $s(Z)>t$, so $e_{k-1}\sim\mathrm{Bin}(k-1,\Rsel(t))$, and a union bound over $\mathcal K$. Implementations often count the boundary trace as well, testing $(e_k,k)$ and accepting $s\ge\lamhat$; that variant needs $P(Y=0\mid s(Z)=t)\ge\Rsel(t)$ for all $t$ and is compared in Appendix~\ref{app:exact}.

\paragraph{Alternatives.} (i) The same construction with the CP limit replaced by the Hoeffding--Bentkus p-value \citep{bates2021rcps,bentkus2004hoeffding,angelopoulos2021ltt} or a mixture betting e-value \citep{waudby2024betting}, on the same grid with the same split. (ii) Conformal selection \citep{jin2023selection}: p-values $p_j=(1+\#\{i:Y_i=0,\,s_i\ge s_j\})/(n_0+1)$ against the $n_0$ incorrect calibration traces, then Benjamini--Hochberg \citep{benjamini1995controlling} at level $\alpha$, which controls the false discovery rate among accepted traces. Conformal risk control \citep{angelopoulos2024crc} does not apply directly, since selective error is not monotone in $\lambda$.

\paragraph{Protocol.}\footnote{Code, data and a one-command reproduction of every table and figure: \url{https://github.com/ArjunBalaji79/certified-by-abstention-code}.} Problems are split 30/40/30\% into train (scorer fitting), calibration and test, \emph{by problem}. Traces from one problem are dependent, so the unit of analysis is \emph{one trace per problem}, redrawn in every replicate; calibration and test units are then exchangeable. Each configuration is evaluated over 1{,}000 replicates (200 resplits for each of 5 scorer seeds). The relative targets $\alpha=\rho\pizero$ use the base error rate of the \emph{training} partition of the same seed, so every target is fixed before any calibration trace is seen; in the experimental sections $\pizero$ denotes this estimate (five values per model, spanning about $\pm0.02$), and the held-out base error serves only to classify targets as vacuous or not. For each we report mean coverage over all replicates, $\pfire$, and $\hat R$, the error rate among accepted test traces pooled over firing replicates (a coverage-weighted estimate of $E[\Rsel(\lamhat)\mid\text{fire}]$).

\section{Three Structural Facts}
\label{sec:theory}

The statements below have short proofs. We give them because they account for most of what we observe and, to our knowledge, have not been stated for verifier certification.

\begin{lemma}[Certification floor]
\label{lem:floor}
Rule \eqref{eq:cert} can fire only if some $k\in\mathcal K$ satisfies $k-1\ge a^\star_0$, where
\[
a^\star_0(\alpha,\delta/G):=\left\lceil \frac{\log (G/\delta)}{-\log(1-\alpha)}\right\rceil ,
\]
and this $k$ suffices only if the $k-1$ calibration traces above $s_{(k)}$ contain no error. Certifying an accepted set with $e$ errors requires $k-1\ge a^\star_e$, the smallest $m$ with $\Ucp(e,m;\delta/G)\le\alpha$, which increases with $e$.
\end{lemma}

\begin{proof}
$\Ucp(0,m;\delta')=1-\delta'^{1/m}\le\alpha$ iff $m\ge\log\delta'/\log(1-\alpha)$. Monotonicity in $e$ follows from the stochastic ordering of $\mathrm{Beta}(e+1,k-e)$ in $e$.
\end{proof}

At $\alpha=\delta=0.10$, $G=20$ the floor is $a_0^\star=51$ error-free traces above the threshold. For a model with $\pizero=0.0886$ calibrated on 74 problems (our Qwen3-8B cell) at $\alpha=\pizero$, the floor is 58: a score with AUC near $0.6$ must place none of about seven calibration errors among the 58 traces above the threshold, 78\% of the calibration set. Figure~\ref{fig:floor} plots $a^\star$ for accepted sets with error fraction $0$, $\alpha/4$ and $\alpha/2$: at half the target, $\alpha=0.10$ needs about 195 accepted calibration points, more than any calibration set here; dropping the Bonferroni grid (dashed) lowers the floor by a factor of about $2.3$.

\begin{figure}[t]
\centering
\includegraphics[width=\columnwidth]{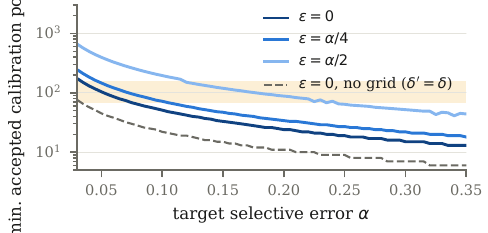}
\caption{Certification floor (Lemma~\ref{lem:floor}): minimum number of calibration traces above the threshold for the CP certificate to certify $\alpha$ ($\delta=0.10$, $G=20$) when a fraction $\varepsilon$ of them are errors. Shaded: calibration-set sizes in this study (67--158). Analytic; no data.}
\label{fig:floor}
\end{figure}

\begin{proposition}[Lattice condition for conformal selection]
\label{prop:lattice}
Let $p_1,\dots,p_m$ be conformal p-values against $n_0$ calibration nulls, with empirical c.d.f.\ $\hat F_m$. BH at level $\alpha$ makes either no rejection or at least $\lceil m/(\alpha(n_0+1))\rceil$ rejections, and it rejects if and only if
\[
\exists\,k\in\{1,\dots,n_0+1\}:\quad \hat F_m\!\Big(\frac{k}{n_0+1}\Big)\ \ge\ \frac{k}{\alpha\,(n_0+1)} .
\]
In particular, if a fraction $q\ge 1/(\alpha(n_0+1))$ of test traces outscores every calibration null, BH rejects.
\end{proposition}

The proof (Appendix~\ref{app:proofs}) uses only that conformal p-values lie on the lattice $\{k/(n_0+1)\}$; the same granularity governs the finite-calibration behaviour of BH with empirical null p-values analysed by \citet{mary2022semisupervised}, whose FDR bounds are exact when $\alpha(n_0+1)/m$ is an integer and whose power analysis likewise needs $n_0\gtrsim m/\max(1,k)$ for $k$ detectable alternatives. With $n_0=25$ and $\alpha=0.10$, BH must accept at least $38.5\%$ of the test set at once or nothing. The requirement $n_0\gtrsim1/(\alpha q)$ is perverse for verification: it is a lower bound on the number of calibration \emph{errors}, it grows as the target tightens, and a more accurate model both forces a tighter non-vacuous target and supplies fewer errors.

\begin{proposition}[Validity by abstention]
\label{prop:abstain}
If a procedure is $(\alpha,\delta)$-valid and fires with probability $\pfire>0$, then
\[
P\big(\Rsel(\lamhat)>\alpha\ \big|\ \text{fire}\big)\ \le\ \min\{1,\ \delta/\pfire\},
\]
and for $\pfire\le\delta$ this is attained: there are $(\alpha,\delta)$-valid procedures whose every issued certificate is wrong.
\end{proposition}

\begin{proof}
$\Rsel(\infty)=0$, so $\{\Rsel(\lamhat)>\alpha\}\subseteq\{\text{fire}\}$ and the bound follows on dividing by $\pfire$. For attainment take any distribution with $\pizero>\alpha$ and the procedure that ignores the data, accepts everything with probability $\pfire\le\delta$, and otherwise abstains.
\end{proof}

Proposition~\ref{prop:abstain} says nothing against validity: validity is a statement about the calibration draw, abstentions included, whereas an operator acts only on certificates that are issued. When $\pfire\le\delta$ the guarantee is compatible with every issued certificate being wrong, a second sense, beyond $\alpha\ge\pizero$, in which it can be vacuous. The statement is elementary, and that is the point: certificates are reported by coverage averaged over calibration draws, or by error averaged with empty certificates counted as zero, and neither statistic can expose this regime; $\pfire$ and the error of issued certificates do.

\section{Attainability}
\label{sec:attain}

\paragraph{Data.} Seven models (Qwen3.5-0.8B/4B, Qwen3-1.7B/4B/8B, Qwen2.5-7B-Instruct, OLMo-3-7B) generate 2--4 traces per GSM8K test problem \citep{cobbe2021gsm8k} under a plain two-shot prompt: 295--1{,}387 usable traces over 167--395 problems per model, hence 67--158 calibration problems. Accuracy ranges from $0.33$ to $0.91$.

\begin{table*}[t]
\centering\scriptsize
\caption{Rule~\eqref{eq:cert} on GSM8K, $\delta=0.10$, 1{,}000 replicates per cell. $n_{\mathrm{cal}}$: calibration problems; $\hat\pi_0$: training-partition base error (sets the targets); $\pi_0^{\mathrm{held}}$: base error of the remaining problems. Per signal: mean coverage, $P_{\mathrm{fire}}$ and $\hat R/\alpha$ (error among accepted test traces pooled over firing replicates, relative to the target; -- if fewer than 50 replicates fire) at $\alpha=\hat\pi_0$, and coverage and $P_{\mathrm{fire}}$ at $0.75\hat\pi_0$. Test AUCs: Appendix Table~\ref{tab:bh}.}
\label{tab:main}
\setlength{\tabcolsep}{2.6pt}
\begin{tabular}{lrrr rrrrr rrrrr rrrrr}
\toprule
 & & & & \multicolumn{5}{c}{Confidence} & \multicolumn{5}{c}{J-lens} & \multicolumn{5}{c}{Probe} \\
\cmidrule(lr){5-9}\cmidrule(lr){10-14}\cmidrule(lr){15-19}
 & & & & \multicolumn{3}{c}{$\alpha=\hat\pi_0$} & \multicolumn{2}{c}{$0.75\hat\pi_0$} & \multicolumn{3}{c}{$\alpha=\hat\pi_0$} & \multicolumn{2}{c}{$0.75\hat\pi_0$} & \multicolumn{3}{c}{$\alpha=\hat\pi_0$} & \multicolumn{2}{c}{$0.75\hat\pi_0$} \\
\cmidrule(lr){5-7}\cmidrule(lr){8-9}\cmidrule(lr){10-12}\cmidrule(lr){13-14}\cmidrule(lr){15-17}\cmidrule(lr){18-19}
Model & $n_{\mathrm{cal}}$ & $\hat\pi_0$ & $\pi_0^{\mathrm{held}}$ & Cov. & $P_{\mathrm{fire}}$ & $\hat R/\alpha$ & Cov. & $P_{\mathrm{fire}}$ & Cov. & $P_{\mathrm{fire}}$ & $\hat R/\alpha$ & Cov. & $P_{\mathrm{fire}}$ & Cov. & $P_{\mathrm{fire}}$ & $\hat R/\alpha$ & Cov. & $P_{\mathrm{fire}}$ \\
\midrule
Qwen3.5-0.8B & 158 & 0.64 & 0.68 & 0.22 & 0.578 & 0.83 & 0.01 & 0.045 & 0.08 & 0.251 & 0.91 & 0.00 & 0.013 & 0.49 & 0.963 & 0.79 & 0.06 & 0.380 \\
Qwen3.5-4B & 154 & 0.16 & 0.16 & 0.39 & 0.571 & 0.63 & 0.11 & 0.220 & 0.18 & 0.239 & 0.79 & 0.02 & 0.045 & 0.72 & 0.934 & 0.53 & 0.39 & 0.618 \\
OLMo-3-7B & 67 & 0.34 & 0.34 & 0.33 & 0.533 & 0.67 & 0.10 & 0.192 & 0.08 & 0.160 & 0.90 & 0.01 & 0.021 & 0.34 & 0.538 & 0.68 & 0.09 & 0.208 \\
Qwen3-1.7B & 78 & 0.21 & 0.22 & 0.09 & 0.115 & 0.92 & 0.01 & 0.012 & 0.04 & 0.054 & 1.07 & 0.00 & 0.006 & 0.21 & 0.293 & 0.72 & 0.05 & 0.094 \\
Qwen2.5-7B-It & 71 & 0.22 & 0.18 & 0.15 & 0.187 & 0.76 & 0.01 & 0.016 & 0.16 & 0.240 & 0.73 & 0.03 & 0.061 & 0.23 & 0.313 & 0.68 & 0.03 & 0.058 \\
Qwen3-4B & 153 & 0.10 & 0.09 & 0.20 & 0.250 & 0.80 & 0.03 & 0.052 & 0.15 & 0.194 & 0.85 & 0.02 & 0.029 & 0.31 & 0.387 & 0.68 & 0.08 & 0.114 \\
Qwen3-8B & 74 & 0.09 & 0.09 & 0.02 & 0.020 & -- & 0.00 & 0.002 & 0.01 & 0.016 & -- & 0.00 & 0.001 & 0.04 & 0.051 & 0.93 & 0.00 & 0.005 \\
\bottomrule
\end{tabular}
\end{table*}

\paragraph{The boundary target.} Table~\ref{tab:main} reports the certificate at $\alpha=\pizero$ and at $0.75\pizero$. The boundary target sits at the edge of vacuity, so it is sensitive to the estimate that defines it: the training-partition $\pizero$ lies above the held-out base error by $0.01$--$0.04$ for three models, below it by $0.01$--$0.03$ for two, and within $0.005$ for two, and coverage at this target moves accordingly. Mean coverage conceals a bimodal outcome: the certificate either abstains or accepts a large set (32--86\% of test traces given firing), and $\pfire$ ranges from $0.02$ to $0.96$ across the 21 pairs. When the certificate fires, $\hat R$ is $0.53$--$1.07$ of the target, close to $1$ for weak signals, which fire by accepting almost everything. The probe covers more than confidence and more than the J-lens signal in all seven models (geometric-mean ratios $1.8$ and $3.3$; sign test $p=0.016$ for each, both significant under BH at 0.10 over the six secondary comparisons of this section and \S\ref{sec:shift}); in absolute terms its coverage exceeds confidence's by up to $0.33$ and the J-lens signal's by up to $0.55$. Confidence covers more than the J-lens signal in six of seven models (ratio $1.9$, $p=0.125$).

\paragraph{Non-vacuous targets.} At $0.75\pizero$, which lies below the held-out base error in every cell, the picture is stark: $\pfire\le\delta$ for 15 of 21 pairs, mean coverage reaches $10\%$ for 2 of 21 (confidence and the probe on Qwen3.5-4B, $0.11$ and $0.39$), and the probe's advantage widens (ratios $3.1$ over confidence and $8.5$ over the J-lens signal). Over all targets below the boundary (Appendix Table~\ref{tab:full}), $10\%$ coverage is reached by 3 of 21 pairs, the probe on Qwen3.5-4B and on Qwen3-1.7B and confidence on Qwen3.5-4B; two more, confidence and the probe on Qwen3.5-0.8B, reach it at the boundary target, which on that model lies below the held-out base error. At $\alpha=0.05$ every signal has coverage below $0.1\%$, as Lemma~\ref{lem:floor} predicts: $a_0^\star=104$ error-free traces above the threshold exceeds four of the seven calibration sets and is two thirds of the largest.

\paragraph{Conformal selection.} At level $\alpha=\pizero$ BH conformal selection makes no rejection in 74\% of replicates on average (Appendix Table~\ref{tab:bh}). The lattice minimum of Proposition~\ref{prop:lattice} ranges from 2\% of the test set (Qwen3.5-0.8B, with 108 calibration errors) to 69\% (Qwen3-4B), and for Qwen3-8B it exceeds the test set, so no rejection is possible. The condition also explains the one success: on Qwen3.5-0.8B, whose calibration sets contain 108 errors, BH rejects in 79--100\% of replicates. The binding quantity is the number of calibration nulls, which for a correctness verifier is the number of calibration \emph{errors}: 6--24 here for six of seven models, against the hundreds of nulls behind the power reported by \citet{gui2024conformal}.

\subsection{Is Low Coverage the Bound's Fault?}
\label{sec:tighter}

No. For fixed $k$ the error count is binomial, a family with monotone likelihood ratio, so the binomial test behind the CP limit is uniformly most powerful among non-randomized tests at its size (Karlin--Rubin); Appendix Table~\ref{tab:bounds} confirms that neither a Hoeffding--Bentkus nor a betting bound exceeds CP in any of the 21 cells.

What does cost coverage is the Bonferroni grid. A walk that tests every size at the full $\delta$, starting at the first size that passes, raises mean coverage from $0.21$ to $0.43$ but has no validity proof, since its start is chosen by the data; the textbook fixed-sequence walk from the strictest threshold \citep{bauer1991multiple} never fires for $\alpha<1-\delta^{1/9}=0.23$ (Lemma~\ref{lem:floor}). Section~\ref{sec:fs} resolves this.

\subsection{Conditional Failure With Known Risk}
\label{sec:synthetic}

Real test sets are too small to tell whether an issued certificate is wrong, so we simulate a Gaussian score model (incorrect scores $N(0,1)$, correct scores $N(\mu,1)$ with $\mu$ set by the AUC), where $\Rsel(\lambda)$ is known, at sizes matched to four of our cells, 20{,}000 calibration draws each (Appendix Table~\ref{tab:synthetic} and Figure~\ref{fig:abstention}). At $\alpha=0.75\pizero$ with $\pizero=0.16$, AUC $0.82$, $n=154$ (the Qwen3.5-4B probe regime), rule \eqref{eq:cert} fires in $56\%$ of draws and an issued certificate is wrong with probability $0.003$. At AUC $0.71$ and $\alpha=0.5\pizero$ it fires in $0.5\%$ of draws and is wrong in $52\%$ of them; at $\pizero=0.22$, AUC $0.76$, $n=78$ and the same target it fires in $0.4\%$ and is wrong in $69\%$. In every setting the unconditional failure probability is at most $0.003$, more than thirty times below $\delta$. The failures are not gross: among failing draws the median $\Rsel(\lamhat)/\alpha$ is $1.04$--$1.12$ and the 90th percentile $1.13$--$1.28$ in the seven settings with at least 29 failures; the eighth fires twice in 20{,}000 draws, both times wrong, at $1.54$ and $1.63$ (Appendix Table~\ref{tab:synthetic}). They need not stay mild: under best-of-32 selection the error among accepted traces reaches $1.12$--$1.52$ of the target in three of four firing pairs while the failure frequency stays below $\delta$ (\S\ref{sec:pressure}). Validity bounds the unconditional failure probability; the operator experiences the conditional one; abstention reconciles the two (Proposition~\ref{prop:abstain}).

\subsection{Score Quality, Not Model Quality}
\label{sec:auc}

One might expect better models to be harder to certify, since a non-vacuous $\alpha$ lies below $\pizero$. Across the seven models, coverage at $\alpha=\pizero$ instead tracks the score's test AUC (Spearman $+0.79$, $p=0.036$, for both probe and confidence) more than model accuracy ($-0.43$ and $-0.39$, both ns). With seven models this is descriptive; budgets should be planned from the verifier's AUC and Lemma~\ref{lem:floor}.

\paragraph{Why the probe wins.} The probe's advantage could be information the readable features cannot express, or better estimation of the same information. A cross-fitted reconstruction of the probe score from the 119 readable features (Appendix~\ref{app:recon}) has mean held-out AUC $0.63$ against $0.75$ for the probe, and the residual has mean AUC $0.66$ ($0.59$--$0.72$ across models), so the edge is not linearly recoverable from the readable features and the coverage gap reflects information rather than tuning, at least at these training sizes.

\section{A Floor-Started Fixed-Sequence Certificate}
\label{sec:fs}

Rule \eqref{eq:cert} pays for a Bonferroni split because its thresholds are order statistics of the calibration scores, so the hypotheses it tests are random. Learn-then-Test \citep{angelopoulos2021ltt} already contains the remedy for a fixed grid: fixed-sequence testing of the grid hypotheses controls the family-wise error with no split and no monotonicity requirement. Two choices separate that recipe from a usable verifier certificate. The thresholds must be fixed without calibration data, which we do from the training split; and the walk must begin where certification is possible at all, since sizes below the floor of Lemma~\ref{lem:floor} cannot certify and a walk that begins there stops at once. FS-floor is fixed-sequence Learn-then-Test with these two choices and a start constant $c$.

\paragraph{Procedure (FS-floor).} Let $\lambda_1>\dots>\lambda_J$ be the $1-q_j$ quantiles of the scorer's \emph{training-split} scores for accept fractions $q_j\in\{0.02,0.04,\dots,1\}$. Let $j_0$ be the first $j$ with $q_j\,n\ge c\,a_0^\star(\alpha,\delta)$. For $j=j_0,j_0+1,\dots$ let $N_j$ and $E_j$ count calibration traces and errors with $s\ge\lambda_j$, and reject $H_j:\Rsel(\lambda_j)>\alpha$ when the exact binomial p-value $P\{\mathrm{Bin}(N_j,\alpha)\le E_j\}$ is at most $\delta$. Stop at the first non-rejection and return the last rejected $\lambda_j$, or abstain if there is none.

\begin{proposition}
\label{prop:fs}
FS-floor is $(\alpha,\delta)$-valid for any score distribution.
\end{proposition}

\begin{proof}
The grid, the order and $j_0$ depend only on the training split and $n$, so each $H_j$ is a fixed hypothesis, and given $N_j$, $E_j\sim\mathrm{Bin}(N_j,\Rsel(\lambda_j))$, so its p-value is valid. A false certification requires rejecting the first true $H_j$ in the order, which has probability at most $\delta$ \citep{bauer1991multiple}. No monotonicity of $\Rsel$ is used.
\end{proof}

\begin{table}[t]
\centering\scriptsize
\caption{FS-floor (cross-fitted grid, $c$ chosen before calibration by simulation) against the Bonferroni CP rule~\eqref{eq:cert}, $\delta=0.10$. Top: Gaussian score model with known risk at the given AUC and calibration size $n$, 20{,}000 draws; $P_{\mathrm{fail}}=P(R(\hat\lambda)>\alpha)$. Bottom: means over 7 models $\times$ 3 signals, 1{,}000 replicates each; $^{\ast}$median of the per-cell choices of $c$.}
\label{tab:fs}
\setlength{\tabcolsep}{2.8pt}
\begin{tabular}{rrl cc cccc}
\toprule
 & & & \multicolumn{2}{c}{CP, Eq.~\eqref{eq:cert}} & \multicolumn{4}{c}{FS-floor} \\
\cmidrule(lr){4-5}\cmidrule(lr){6-9}
AUC & $n$ & $\alpha$ & Cov. & $P_{\mathrm{fire}}$ & $c$ & Cov. & $P_{\mathrm{fire}}$ & $P_{\mathrm{fail}}$ \\
\midrule
0.82 & 154 & $0.75\pi_0$ & 0.321 & 0.55 & 4 & 0.635 & 0.85 & 0.053 \\
0.82 & 154 & $0.5\pi_0$ & 0.051 & 0.09 & 2 & 0.272 & 0.49 & 0.039 \\
0.71 & 154 & $0.75\pi_0$ & 0.052 & 0.11 & 3 & 0.231 & 0.40 & 0.040 \\
0.71 & 154 & $0.5\pi_0$ & 0.003 & 0.01 & 1.5 & 0.045 & 0.12 & 0.022 \\
0.76 & 78 & $0.75\pi_0$ & 0.061 & 0.12 & 3 & 0.262 & 0.40 & 0.043 \\
0.66 & 153 & $0.75\pi_0$ & 0.010 & 0.02 & 2 & 0.083 & 0.14 & 0.023 \\
\midrule
\multicolumn{2}{l}{real, 21 pairs} & $\pi_0$ & 0.211 & 0.33 & 4$^{\ast}$ & 0.390 & 0.50 & -- \\
\multicolumn{2}{l}{real, 21 pairs} & $0.75\pi_0$ & 0.051 & 0.10 & 3$^{\ast}$ & 0.163 & 0.28 & -- \\
\multicolumn{2}{l}{real, 21 pairs} & $0.5\pi_0$ & 0.004 & 0.01 & 1.5$^{\ast}$ & 0.033 & 0.07 & -- \\
\bottomrule
\end{tabular}
\end{table}

\paragraph{Choosing $c$ and the grid.} The constant $c$ sets how far above the zero-error floor the walk begins. Starting at the floor itself ($c=1$) makes the first test intolerant of a single error, and on real data at $\alpha=\pizero$ that covers less than rule \eqref{eq:cert} ($0.13$ against $0.21$). We therefore choose $c\in\{1,1.5,2,3,4\}$ per cell \emph{before calibration}: we simulate the Gaussian score model of \S\ref{sec:synthetic} at the training split's out-of-fold AUC, its base error rate and $n$, and take the $c$ with the largest simulated coverage. The choice uses no calibration data, so Proposition~\ref{prop:fs} is untouched, and the Gaussian assumption affects only power. The grid is cross-fitted in the same spirit: its quantiles come from out-of-fold scores of the training problems, since in-sample scores are optimistic for the scorer. Appendix Table~\ref{tab:fsfull} and Figure~\ref{fig:attain} report every fixed $c$, the in-sample grid and the per-model curves.

\paragraph{Results.} Table~\ref{tab:fs} compares FS-floor with rule \eqref{eq:cert}. In simulation the failure probability is at most $0.055$ against $\delta=0.10$ for every $c$ and setting. At the chosen $c$, coverage is $2.0$--$8.3\times$ that of rule \eqref{eq:cert} in the five settings where the latter fires in more than $1\%$ of draws, and $15\times$ in the one where it fires in $1\%$; and $\pfire$ rises from $0.11$ to $0.40$ in the AUC-$0.71$ regime; the probability that an issued certificate is wrong is $0.06$--$0.11$ in the four settings where $\pfire$ is at least $0.4$ and grows as $\pfire$ falls, as Proposition~\ref{prop:abstain} predicts. FS-floor does not control this conditional failure, which validity alone does not guarantee; by raising $\pfire$ it tightens the bound $\delta/\pfire$. On the real cells FS-floor covers more than rule \eqref{eq:cert} in all 21 pairs at all three targets. At $\alpha=0.75\pizero$ mean coverage rises from $0.051$ to $0.163$ and $\pfire$ from $0.10$ to $0.28$, the number of pairs with $\pfire\le\delta$ falls from 15 to 4, and 11 pairs reach $10\%$ coverage against 2; at $\alpha=\pizero$ coverage rises from $0.211$ to $0.390$ and all 21 pairs reach $10\%$ against 14. Conditional on firing, $\hat R/\alpha$ has median $0.80$ at $\alpha=\pizero$ and $0.94$ at $0.75\pizero$, against $0.76$ and $0.84$ for rule \eqref{eq:cert}: a procedure that abstains less certifies sets closer to the target, as it should. Two cheap baselines locate the gain (Appendix Table~\ref{tab:baselines}): shrinking the Bonferroni grid of rule \eqref{eq:cert} to $G=5$ raises mean coverage at $0.75\pizero$ only from $0.05$ to $0.07$, and the floor start alone, a single fixed threshold tested at level $\delta$, reaches $0.10$, and the walk $0.16$. The simulation that chooses $c$ also predicts $\pfire$ before any calibration label is seen; the forecast is rough (mean absolute error $0.10$ against a mean $\pfire$ of $0.28$ at $0.75\pizero$) but ranks the 21 pairs well (Spearman $0.70$), so an operator can tell before labelling whether a certificate is likely to fire. The floor remains binding: no non-randomized procedure on this statistic can certify with fewer than $a_0^\star(\alpha,\delta)$ accepted points, and mean coverage at $0.75\pizero$ stays at $0.16$. Figure~\ref{fig:budget} shows what larger budgets buy in the Gaussian model: at AUC $0.76$ the certificate fires in most draws from about 200 problems and covers half the traces from about 500; at AUC $0.66$ even 3{,}000 problems certify fewer than half. FS-floor's advantage is concentrated below a few hundred problems.

\begin{figure}[t]
\centering
\includegraphics[width=\columnwidth]{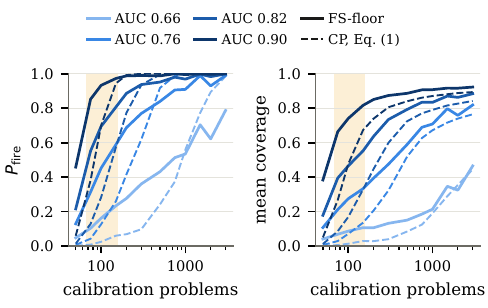}
\caption{Calibration budget curves in the Gaussian score model ($\pizero=0.16$, $\alpha=0.75\pizero$, $\delta=0.10$, 2{,}000 draws per point): firing probability and mean coverage of FS-floor ($c=2$, solid) and rule \eqref{eq:cert} (dashed) against the number of calibration problems, for four values of the verifier's AUC. Shaded: the budgets in this study.}
\label{fig:budget}
\end{figure}

\section{Shift}
\label{sec:shift}

We freeze the GSM8K-trained scorer and the GSM8K-calibrated threshold ($\alpha=\pizero$ of GSM8K) and apply them to the same model's traces on ARC-Challenge \citep{clark2018arc}, StrategyQA \citep{geva2021strategyqa}, CommonsenseQA \citep{talmor2019commonsenseqa} and three BIG-Bench-Hard logic tasks \citep{suzgun2023bbh}: 72 (model, benchmark, signal) cells. Exchangeability fails by construction \citep{tibshirani2019covariate,barber2023beyond}.

Every one of the 72 cells fires in some replicates. In 40 of them $\hat R$ exceeds the target, by up to $0.28$, and the split is clean: the target is exceeded in 40 of the 45 cells where the new task is harder than the target and in none of the 27 where it is easier. The error among accepted traces is a median $0.94$ of the new task's base error, so off-distribution the accepted set is barely better than accepting everything; Appendix Figure~\ref{fig:shift} and Table~\ref{tab:shift} show every cell. Across cells $\hat R$ tracks the new task's base error (Spearman $0.88$), and the remaining benefit shrinks as the score's AUC drops (Spearman $0.42$; all $p<10^{-3}$, nominal since cells cluster within models).

\section{Selection Pressure}
\label{sec:pressure}

\begin{figure*}[t]
\centering
\includegraphics[width=\textwidth]{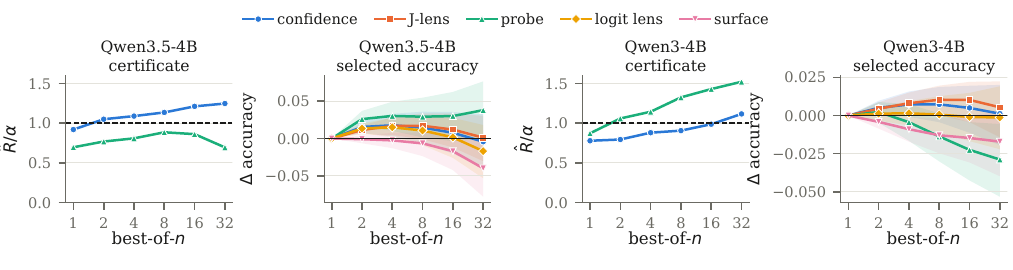}
\caption{Best-of-$n$ selection against each verifier signal. \emph{Certificate panels}: error among accepted BoN-selected test traces, pooled over firing replicates, relative to $\alpha=\pizero$; thresholds calibrated on unpressured traces. Signals that fire in fewer than 2\% of 1{,}000 replicates are omitted. \emph{Accuracy panels}: change in accuracy of the selected trace relative to $n=1$, with 95\% bootstrap bands over problems.}
\label{fig:pressure}
\end{figure*}

A deployed verifier is an optimization target: candidates are sampled and the top-scoring one is kept. Marginal conformal guarantees can fail under data-driven selection \citep{jin2025focal}, and proxy scores are overoptimized by best-of-$n$ (BoN) \citep{gao2023scaling}. We measure the dose-response for verifier certificates.

\paragraph{Design.} For Qwen3.5-4B and Qwen3-4B we sample up to 64 traces per problem on GSM8K problems unused elsewhere in this paper, so the frozen scorers have never seen them; 149 and 243 problems retain at least 32 usable traces, and analyses use 32 per problem. Pressure runs over $n\in\{1,2,4,\dots,32\}$, up to $\log n-(n-1)/n=2.5$ nats from the base policy \citep{beirami2024bon}. In the \emph{naive arm} the threshold is calibrated on unpressured traces from half of the problems and applied to BoN-selected traces of the other half. The target is the model's GSM8K training-partition base error from \S\ref{sec:attain} ($0.16$ and $0.10$), fixed before any pressure problem was drawn; the sample base errors of the pressure problems are $0.20$ and $0.11$, so both targets are non-vacuous there. In the \emph{corrected arm} calibration traces are selected by the same BoN rule, restoring exchangeability conditional on selection \citep{jin2025focal}. We use 1{,}000 resplits.

\paragraph{Error among accepted rises with $n$.} Calibration sets here are 74 and 122 problems, half of the pressure problems, and the targets are non-vacuous, so $\pfire$ is low: $0.14$ for the probe and $0.03$ for confidence on Qwen3.5-4B, $0.09$ and $0.05$ for confidence and the probe on Qwen3-4B, and at most $0.02$ otherwise. Conditional on firing (Figure~\ref{fig:pressure}; Appendix Table~\ref{tab:bon} and Figure~\ref{fig:corrected}), $\hat R/\alpha$ rises from $0.87$ to $1.52$ for the probe and from $0.78$ to $1.12$ for confidence on Qwen3-4B, and from $0.92$ to $1.25$ for confidence on Qwen3.5-4B. The exception is the probe on Qwen3.5-4B ($0.70\to0.69$), the one pair for which BoN improves the policy monotonically. The naive arm breaks exchangeability between calibration and deployment, so Proposition~\ref{prop:valid} promises nothing for it; what happens empirically is that the fraction of replicates whose test error among accepted exceeds $\alpha$ stays below $\delta$ in every cell (at most $0.056$), because most replicates accept nothing.

\paragraph{Overoptimization.} The accuracy of the BoN-selected trace (Figure~\ref{fig:pressure}, accuracy panels) peaks at small $n$ and then falls for several signals. On Qwen3-4B the probe loses $3.3$ points from its peak at $n=2$ and confidence $1.8$ from $n=4$; on Qwen3.5-4B the surface signal loses $5.7$ and the logit lens $2.6$, while the probe gains $7.7$ monotonically. No signal is uniformly safe; the bands are wide and we make no confirmatory claim about the ranking. Regardless of ranking, a bound on error among accepted traces, calibrated before optimization, says nothing about whether selecting on the score still improves the policy.

\paragraph{The correction costs coverage.} Selection-conditional calibration restores the exchangeability that Proposition~\ref{prop:valid} needs, so the corrected arm is valid at every $n$, but at $n=32$ its mean coverage falls to $0.005$ or below in three of the four principal cells. Only the strongest scorer benefits (probe on Qwen3.5-4B, $0.09\to0.22$), because BoN has lowered that policy's error to $0.12$, below the $0.16$ target.

\section{Related Work}

\paragraph{Distribution-free selection and risk control.} Selective classification with binomial certificates goes back to \citet{elyaniv2010foundations} and \citet{geifman2017selective}; RCPS \citep{bates2021rcps}, Learn-then-Test \citep{angelopoulos2021ltt} and conformal risk control \citep{angelopoulos2024crc} generalize the risk and the search. Conformal p-values with BH give FDR-controlled selection \citep{bates2023outliers,jin2023selection}, applied to LLM outputs by \citet{gui2024conformal,mohri2024factuality,quach2024conformal}; these works show power at moderate to large calibration sizes, and we characterize the small-budget regime.

\paragraph{Verifier signals.} Confidence \citep{kadavath2022know}, hidden-state probes \citep{azaria2023internal,burns2023discovering,zhang2025reasoning}, learned verifiers \citep{cobbe2021gsm8k,lightman2024verify} and lens readouts \citep{nostalgebraist2020logit,belrose2023tuned,gurnee2026workspace} are evaluated by ranking metrics; monitoring written CoT, which can be unfaithful \citep{turpin2023unfaithful}, is an active proposal \citep{korbak2025monitorability,baker2025monitoring}. We price readability in deployment terms.

\paragraph{Overoptimization.} \citet{gao2023scaling} characterize BoN and RL overoptimization of reward models; ensembles mitigate but do not remove it \citep{coste2024ensembles,eisenstein2023herding}. We study correctness verifiers with exact labels and ask what happens to a certificate; \citet{jin2025focal} prove that selection breaks marginal conformal coverage and give the correction whose coverage cost we measure.

\section{Limitations and Discussion}

All models are open-weight decoders of at most 8B parameters; certification is on GSM8K with one prompt format and temperature $0.8$. Scorers are fit on the 30\% training partition, which for the most accurate model holds only 4--12 incorrect traces, so the readable signals' AUCs, and the coverage gap to the probe, may narrow with more training errors. The best-of-$n$ results cover two models. Replicate-based uncertainty reflects resplitting on a fixed problem sample. The probe reconstruction is linear, so a non-linear readable reconstruction could narrow the gap.

We draw four practical conclusions. Report $\pfire$ next to coverage, and treat $(\alpha,\delta)$-validity as uninformative about issued certificates when $\pfire\le\delta$. Fix thresholds before calibration and start testing at the floor. Budget calibration labels from Lemma~\ref{lem:floor} and the verifier's AUC before collecting them; accepted traces are only a fraction of the labelled ones. Do not read a valid certificate as evidence that the verifier is safe to optimize against or to move to a new task: in our data the error among accepted traces followed the task and the optimizer, not the certificate.

\subsection*{AI use statement}

We used a large language model coding assistant throughout this work. It helped write and debug the analysis code, draft and edit the text and figures, and refine the experimental design, hypotheses, proofs of the statements in Section~\ref{sec:theory}, data processing and interpretation of results. The research questions, design decisions and claims are the authors'; we checked every script and every line of the paper, all tables and figures are regenerated by script from the released result files, and we take responsibility for the final content.

\bibliography{refs}

\clearpage
\appendix
\onecolumn
\begin{center}{\Large\bfseries Appendix}\end{center}
\vspace{0.5em}

\section{Proofs}
\label{app:proofs}

\paragraph{Proposition~\ref{prop:valid}.} Let $(Y_i,S_i)$, $i\le n$, be i.i.d.\ with $S_i$ continuous, and fix $k$. Conditional on $S_{(k)}=t$, the $k-1$ observations with $S_i>t$ are i.i.d.\ from the law of $(Y,S)$ given $S>t$, by the standard conditional structure of order statistics. Hence $e_{k-1}\mid\{S_{(k)}=t\}\sim\mathrm{Bin}(k-1,\Rsel(t))$ with $\Rsel(t)=P(Y=0\mid S>t)$, which for continuous scores equals the selective risk of the acceptance rule $\{s\ge t\}$. The Clopper--Pearson limit satisfies $P\{\Ucp(X,m;\delta')<p\}\le\delta'$ for $X\sim\mathrm{Bin}(m,p)$, so $P\{\Ucp(e_{k-1},k-1;\delta')<\Rsel(S_{(k)})\mid S_{(k)}=t\}\le\delta'$ for every $t$, and therefore unconditionally. The procedure errs only if, for the selected $\hat k$, $\Ucp(e_{\hat k-1},\hat k-1;\delta/G)\le\alpha<\Rsel(S_{(\hat k)})$, an event contained in the union over $k\in\mathcal K$ of the events just bounded, which has probability at most $G\cdot\delta/G=\delta$. When nothing is selected $\Rsel(\infty)=0\le\alpha$.

\paragraph{Boundary-included variant.} If the rule also counts the boundary trace, testing $(e_k,k)$ and accepting $s\ge S_{(k)}$, then $e_k=e_{k-1}+B$ where, given $S_{(k)}=t$, $B\sim\mathrm{Bern}(r(t))$ with $r(t)=P(Y=0\mid S=t)$, independent of $e_{k-1}$. If $r(t)\ge\Rsel(t)$ then $e_k$ is stochastically larger than $\mathrm{Bin}(k,\Rsel(t))$, and since $\Ucp(\cdot,k;\delta')$ is increasing in its first argument the same bound holds. The condition is implied by a monotone likelihood ratio between the score distributions of incorrect and correct traces. Without it the variant can be anti-conservative by at most the contribution of one observation. Table~\ref{tab:bh} compares the two rules on our data. \qed

\paragraph{Proposition~\ref{prop:lattice}.} Write $N=n_0+1$. Every conformal p-value lies in $\{1/N,\dots,1\}$. BH rejects the $R$ smallest p-values, $R=\max\{r:p_{(r)}\le\alpha r/m\}$, with $R=0$ if the set is empty. If $R\ge1$ then $1/N\le p_{(R)}\le\alpha R/m$, so $R\ge m/(\alpha N)$, which gives the first claim. For the equivalence: if $m\hat F_m(k/N)\ge km/(\alpha N)$ for some $k$, put $r=m\hat F_m(k/N)$; then $p_{(r)}\le k/N\le\alpha r/m$, so $R\ge r\ge1$. Conversely, if $p_{(r)}\le\alpha r/m$ for some $r\ge1$, let $k=Np_{(r)}$; then $m\hat F_m(k/N)\ge r\ge p_{(r)}m/\alpha=km/(\alpha N)$. The last claim is the case $k=1$, since $\hat F_m(1/N)$ is the fraction of test traces that outscore every calibration null. \qed

\section{Signals, Data and Compute}
\label{app:signals}

\paragraph{Generation.} Plain-text two-shot prompts with exemplars from the GSM8K training split; temperature $0.8$, top-$p$ $0.95$, at most 640 new tokens; 2--4 samples per problem. A trace is usable if it contains a parseable final answer. For multiple-choice and yes/no benchmarks answers are parsed with a style-anchored rule, because base models frequently continue writing past their answer. Models: \texttt{Qwen3.5-0.8B} and \texttt{Qwen3.5-4B} \citep{qwen2026qwen35}, \texttt{Qwen3-1.7B}, \texttt{Qwen3-4B} and \texttt{Qwen3-8B} \citep{yang2025qwen3}, \texttt{Qwen2.5-7B-Instruct} \citep{qwen2024qwen25} and \texttt{Olmo-3-1025-7B} \citep{olmo2025olmo3}, all under their published open-weight licences. The Jacobian lens and its pre-fitted per-model lenses are the public release accompanying \citet{gurnee2026workspace}.

\paragraph{Features.} Residual-stream states are read at the last token of each reasoning step, at every second layer, in a teacher-forced pass. \emph{J-lens} and \emph{logit lens}: 51 features each, namely step-to-step cosine velocity, curvature, path length, tortuosity and terminal alignment of the decoded top-20 token distributions, plus entropy, churn and the rate of ``silent'' concepts (decoded but never written). \emph{Surface}: 9 geometric features of \texttt{all-mpnet-base-v2} embeddings of the written step text. \emph{Confidence}: 5 statistics of the token log-probabilities. \emph{Length}: 3 trace-length features, included only in the 119-feature readable set of Appendix~\ref{app:recon}. \emph{Probe}: residual vectors at layer $\mathrm{round}(L/2)$, mean over step positions concatenated with the answer position ($2d_{\mathrm{model}}$ dimensions).

\paragraph{Scorers.} Standardization followed by $\ell_2$-regularized logistic regression ($C=1.0$ for the four readable signals, $C=0.01$ for the probe), fit on all traces of the 30\% training problems, separately for 5 seeds of the train split, and frozen. The target levels $\alpha=\rho\,\pizero$ use $\pizero$ estimated on the training partition of the same seed.

\paragraph{Compute.} Generation and feature extraction ran on single A100 GPUs from a commercial cloud provider, for a total cost of roughly US\$50. All certificate analyses run on a laptop CPU from the cached per-trace scores. Code, data and a script that regenerates every table and figure are in the anonymized repository \url{https://github.com/ArjunBalaji79/certified-by-abstention-code} (about 70 minutes end to end on a laptop CPU).

\paragraph{Best-of-$n$ details.} The certificate arms draw best-of-$n$ subsets by Monte Carlo from 32 traces per problem; the accuracy curves use the closed-form rank weights. Exact duplicate traces (6\% and 15\% of samples for the two models) make score ties common; ties are broken by a fixed random key, and tied traces share a label.

\paragraph{Statistics.} For each configuration we report mean coverage over all replicates, the firing probability $\pfire$, and $\hat R$, the error rate among accepted test traces pooled over firing replicates. We do not average per-replicate error with empty certificates counted as zero, since that average is dominated by abstentions. The comparison of J-lens and confidence coverage, the probe residual analysis and the best-of-$n$ sweep were specified before the corresponding data were collected; statistics conditional on firing, the shift analysis by task difficulty and FS-floor were developed afterwards.

\section{Probe Reconstruction}
\label{app:recon}

The probe's advantage could be information that the readable features cannot express, or better estimation of the same information. We reconstruct the probe score from the 119 readable features by a ridge regression fit on out-of-fold probe scores of the training problems, choose every regularization constant by grouped cross-validation on those problems, and evaluate on held-out problems (below). The reconstruction has mean AUC $0.63$ across the seven models, no better than the readable logistic regression ($0.65$) and well below the probe ($0.75$); the residual, probe minus reconstruction, keeps mean AUC $0.66$ and lies between $0.59$ and $0.72$ on every model; and as a certificate signal the reconstruction buys the readable signals' coverage, not the probe's. The probe's edge is therefore not linearly expressible in the readable features, and the coverage gap reflects information rather than tuning, at least at these training sizes: cross-validating the readable scorers' regularization moves their AUC by $0.02$. (Regressing on in-sample probe scores instead gives residual AUC $0.54$; the next paragraph explains why.)

For each of 5 splits by problem (30\% training), we fit the probe ($2d_{\mathrm{model}}$-dimensional mid-depth residual vectors at the step and answer positions, $2{,}048$--$8{,}192$ across models, $C\in\{10^{-3},10^{-2},10^{-1}\}$) and the readable logistic regression (119 features, $C\in\{10^{-2},\dots,10\}$) on the training problems with $C$ chosen by 5-fold grouped cross-validation, obtain out-of-fold probe scores on the training problems, and fit a ridge regression (standardized features, $\lambda\in\{0.1,\dots,100\}$ by the same cross-validation) from the readable features to those out-of-fold scores. All quantities in Table~\ref{tab:recon} are evaluated on the held-out problems. An in-sample variant of this test, which regresses on the probe's \emph{in-sample} scores on the training problems, is misleading: a $2d_{\mathrm{model}}$-dimensional probe separates its training data almost perfectly, so that regression learns a readable predictor of the label rather than of the probe, and its residual (AUC $0.54$) removes the label-correlated part of the probe by construction. The certificate columns rerun rule~\eqref{eq:cert} on the held-out problems (calibration/test split $4{:}3$, 200 resplits per training split) with each score as the signal.

\begin{table}[H]
\centering\small
\caption{Cross-fitted probe reconstruction. Held-out AUC (mean over 5 train/held-out splits by problem) of the probe, the readable logistic regression, a ridge reconstruction of the probe score from the 119 readable features (fit on out-of-fold probe scores of the training problems), and the residual probe $-$ reconstruction; all regularization constants chosen by grouped cross-validation on the training problems. ``Resid.\ (in-sample)'': the residual when the ridge is instead fit on in-sample probe scores, which are nearly separated. Right: mean coverage of the CP certificate at $\alpha=\pi_0$ on held-out problems with each score as the signal.}
\label{tab:recon}
\setlength{\tabcolsep}{4pt}
\begin{tabular}{l ccccc ccc}
\toprule
 & \multicolumn{5}{c}{Held-out AUC} & \multicolumn{3}{c}{Coverage, $\alpha=\pi_0$} \\
\cmidrule(lr){2-6}\cmidrule(lr){7-9}
Model & Probe & Readable & Recon. & Resid. & Resid.\ (in-sample) & Probe & Readable & Recon. \\
\midrule
Qwen3.5-0.8B & 0.812 & 0.700 & 0.738 & 0.721 & 0.621 & 0.714 & 0.485 & 0.590 \\
Qwen3.5-4B & 0.846 & 0.719 & 0.743 & 0.694 & 0.614 & 0.774 & 0.516 & 0.566 \\
OLMo-3-7B & 0.758 & 0.712 & 0.707 & 0.649 & 0.513 & 0.346 & 0.205 & 0.189 \\
Qwen3-1.7B & 0.767 & 0.618 & 0.567 & 0.697 & 0.539 & 0.217 & 0.067 & 0.016 \\
Qwen2.5-7B-It & 0.650 & 0.680 & 0.554 & 0.595 & 0.440 & 0.182 & 0.209 & 0.143 \\
Qwen3-4B & 0.739 & 0.658 & 0.597 & 0.639 & 0.486 & 0.296 & 0.201 & 0.217 \\
Qwen3-8B & 0.701 & 0.445 & 0.525 & 0.647 & 0.538 & 0.013 & 0.002 & 0.004 \\
\midrule
Mean & 0.753 & 0.648 & 0.633 & 0.663 & 0.536 &  &  &  \\
\bottomrule
\end{tabular}
\end{table}

\clearpage
\section{Additional Figures}
\label{app:figs}

\begin{figure}[H]
\centering
\includegraphics[width=\textwidth]{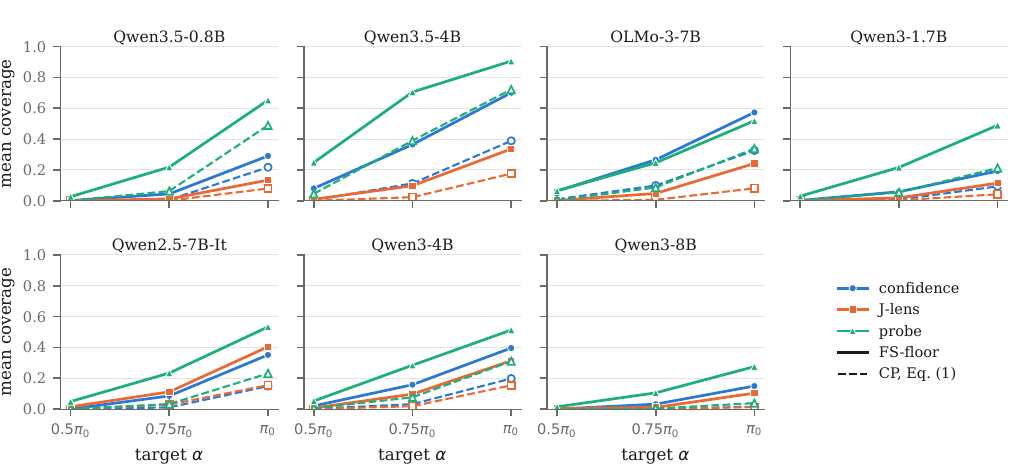}
\caption{Mean coverage at the three relative targets for each model and principal signal on GSM8K: rule~\eqref{eq:cert} (dashed, open markers) and FS-floor with the cross-fitted grid and simulation-chosen $c$ (solid). FS-floor covers more in all 21 pairs at all three targets, and the gap is largest where the certificate fires at all.}
\label{fig:attain}
\end{figure}

\begin{figure}[H]
\centering
\includegraphics[width=0.5\textwidth]{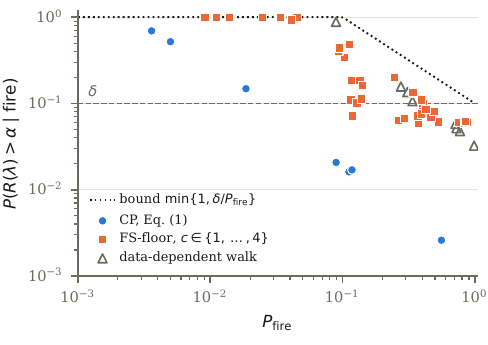}
\caption{Validity by abstention in the Gaussian score model (the settings of Tables~\ref{tab:fs} and~\ref{tab:synthetic}). Each point is one procedure in one setting: the probability that an issued certificate is wrong, against the firing probability. Proposition~\ref{prop:abstain} bounds the former by $\min\{1,\delta/P_{\mathrm{fire}}\}$ (dotted). Rule~\eqref{eq:cert} sits far below $\delta$ where it fires often and far above it where it rarely fires; FS-floor sits near $\delta$; the data-dependent walk of \S\ref{sec:tighter} has no bound.}
\label{fig:abstention}
\end{figure}

\begin{figure}[H]
\centering
\includegraphics[width=0.55\textwidth]{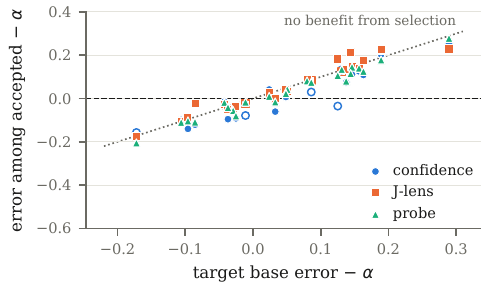}
\caption{GSM8K-calibrated certificates on other benchmarks (Section~\ref{sec:shift}; every cell in Table~\ref{tab:shift}): excess of the error among accepted traces over the target, against how much harder the new task is. Filled: mean coverage above $2\%$. On the dotted diagonal the accepted set is no better than accepting everything.}
\label{fig:shift}
\end{figure}

\begin{figure}[H]
\centering
\includegraphics[width=0.68\textwidth]{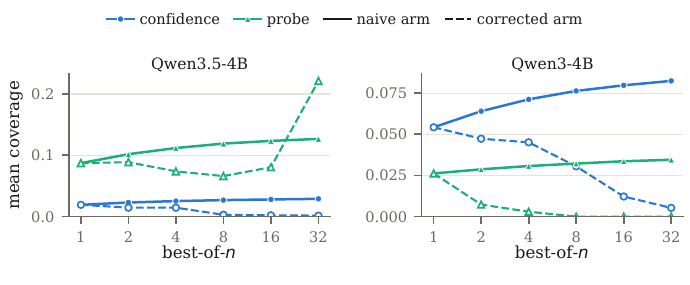}
\caption{Best-of-$n$: mean coverage of the naive arm (solid; threshold calibrated on unpressured traces) and of the selection-conditional arm (dashed; calibration traces selected by the same best-of-$n$ rule), for the signals that fire in at least 2\% of replicates. The correction restores exchangeability, but its coverage collapses with $n$ except for the probe on Qwen3.5-4B, where best-of-$n$ lowers the policy's error below the target.}
\label{fig:corrected}
\end{figure}

\clearpage
\section{Additional Tables}
\label{app:exact}\label{app:auc}

\begin{table}[t]
\centering\small
\caption{Bounds in the same selective construction (same grid, Bonferroni split, data and replicates) at $\alpha=\pi_0$; 7 models $\times$ 3 signals. $^\dagger$No validity proof (see text).}
\label{tab:bounds}
\setlength{\tabcolsep}{4pt}
\begin{tabular}{lccc}
\toprule
Bound & Mean cov. & Cov.$\ge$10\% & Beats CP \\
\midrule
Clopper--Pearson & 0.211 & 14/21 & -- \\
Hoeffding--Bentkus & 0.156 & 12/21 & 0/21 \\
Betting mixture & 0.087 & 6/21 & 0/21 \\
\midrule
Data-dep.\ walk$^\dagger$ & 0.435 & 21/21 & -- \\
\bottomrule
\end{tabular}
\end{table}

\begin{table}[H]
\centering\small
\caption{Mean coverage of the CP certificate (in parentheses: $P_{\mathrm{fire}}$) for every model, target and signal; GSM8K, $\delta=0.10$, 1{,}000 replicates. ``Non-vac.'': target below the base error rate of the non-training problems.}
\label{tab:full}
\begin{tabular}{llrc ccccc}
\toprule
Model & Target & $\alpha$ & Non-vac. & Confidence & J-lens & Logit lens & Surface & Probe \\
\midrule
Qwen3.5-0.8B & $0.5\pi_0$ & 0.322 & yes & 0.000 (0.00) & 0.000 (0.00) & 0.000 (0.00) & 0.000 (0.00) & 0.002 (0.02) \\
 & $0.75\pi_0$ & 0.483 & yes & 0.007 (0.05) & 0.001 (0.01) & 0.000 (0.00) & 0.000 (0.00) & 0.063 (0.38) \\
 & $\pi_0$ & 0.644 & yes & 0.218 (0.58) & 0.080 (0.25) & 0.051 (0.13) & 0.000 (0.00) & 0.489 (0.96) \\
 & 0.05 & 0.050 & yes & 0.000 (0.00) & 0.000 (0.00) & 0.000 (0.00) & 0.000 (0.00) & 0.000 (0.00) \\
 & 0.10 & 0.100 & yes & 0.000 (0.00) & 0.000 (0.00) & 0.000 (0.00) & 0.000 (0.00) & 0.000 (0.00) \\
 & 0.20 & 0.200 & yes & 0.000 (0.00) & 0.000 (0.00) & 0.000 (0.00) & 0.000 (0.00) & 0.000 (0.00) \\
\addlinespace
Qwen3.5-4B & $0.5\pi_0$ & 0.081 & yes & 0.005 (0.01) & 0.000 (0.00) & 0.000 (0.00) & 0.000 (0.00) & 0.046 (0.08) \\
 & $0.75\pi_0$ & 0.122 & yes & 0.113 (0.22) & 0.024 (0.05) & 0.013 (0.03) & 0.001 (0.00) & 0.390 (0.62) \\
 & $\pi_0$ & 0.162 & no & 0.389 (0.57) & 0.176 (0.24) & 0.144 (0.21) & 0.054 (0.07) & 0.721 (0.93) \\
 & 0.05 & 0.050 & yes & 0.000 (0.00) & 0.000 (0.00) & 0.000 (0.00) & 0.000 (0.00) & 0.000 (0.00) \\
 & 0.10 & 0.100 & yes & 0.014 (0.03) & 0.002 (0.01) & 0.001 (0.00) & 0.000 (0.00) & 0.152 (0.28) \\
 & 0.20 & 0.200 & no & 0.727 (0.92) & 0.474 (0.61) & 0.427 (0.56) & 0.159 (0.19) & 0.898 (1.00) \\
\addlinespace
OLMo-3-7B & $0.5\pi_0$ & 0.172 & yes & 0.010 (0.02) & 0.000 (0.00) & 0.000 (0.00) & 0.000 (0.00) & 0.003 (0.01) \\
 & $0.75\pi_0$ & 0.258 & yes & 0.099 (0.19) & 0.007 (0.02) & 0.009 (0.02) & 0.003 (0.01) & 0.087 (0.21) \\
 & $\pi_0$ & 0.344 & no & 0.326 (0.53) & 0.081 (0.16) & 0.098 (0.15) & 0.086 (0.13) & 0.338 (0.54) \\
 & 0.05 & 0.050 & yes & 0.000 (0.00) & 0.000 (0.00) & 0.000 (0.00) & 0.000 (0.00) & 0.000 (0.00) \\
 & 0.10 & 0.100 & yes & 0.000 (0.00) & 0.000 (0.00) & 0.000 (0.00) & 0.000 (0.00) & 0.000 (0.00) \\
 & 0.20 & 0.200 & yes & 0.014 (0.03) & 0.000 (0.00) & 0.000 (0.00) & 0.000 (0.00) & 0.006 (0.02) \\
\addlinespace
Qwen3-1.7B & $0.5\pi_0$ & 0.107 & yes & 0.000 (0.00) & 0.000 (0.00) & 0.000 (0.00) & 0.000 (0.00) & 0.002 (0.00) \\
 & $0.75\pi_0$ & 0.160 & yes & 0.008 (0.01) & 0.003 (0.01) & 0.005 (0.01) & 0.004 (0.01) & 0.055 (0.09) \\
 & $\pi_0$ & 0.213 & yes & 0.094 (0.11) & 0.042 (0.05) & 0.082 (0.10) & 0.051 (0.07) & 0.212 (0.29) \\
 & 0.05 & 0.050 & yes & 0.000 (0.00) & 0.000 (0.00) & 0.000 (0.00) & 0.000 (0.00) & 0.000 (0.00) \\
 & 0.10 & 0.100 & yes & 0.000 (0.00) & 0.000 (0.00) & 0.000 (0.00) & 0.000 (0.00) & 0.000 (0.00) \\
 & 0.20 & 0.200 & yes & 0.023 (0.04) & 0.005 (0.01) & 0.006 (0.01) & 0.001 (0.00) & 0.131 (0.23) \\
\addlinespace
Qwen2.5-7B-It & $0.5\pi_0$ & 0.110 & yes & 0.000 (0.00) & 0.001 (0.00) & 0.000 (0.00) & 0.000 (0.00) & 0.001 (0.00) \\
 & $0.75\pi_0$ & 0.166 & yes & 0.012 (0.02) & 0.030 (0.06) & 0.016 (0.03) & 0.014 (0.02) & 0.033 (0.06) \\
 & $\pi_0$ & 0.221 & no & 0.149 (0.19) & 0.156 (0.24) & 0.118 (0.19) & 0.117 (0.14) & 0.232 (0.31) \\
 & 0.05 & 0.050 & yes & 0.000 (0.00) & 0.000 (0.00) & 0.000 (0.00) & 0.000 (0.00) & 0.000 (0.00) \\
 & 0.10 & 0.100 & yes & 0.000 (0.00) & 0.000 (0.00) & 0.000 (0.00) & 0.000 (0.00) & 0.000 (0.00) \\
 & 0.20 & 0.200 & no & 0.029 (0.04) & 0.064 (0.12) & 0.034 (0.06) & 0.023 (0.04) & 0.089 (0.15) \\
\addlinespace
Qwen3-4B & $0.5\pi_0$ & 0.051 & yes & 0.003 (0.00) & 0.001 (0.00) & 0.000 (0.00) & 0.000 (0.00) & 0.006 (0.01) \\
 & $0.75\pi_0$ & 0.076 & yes & 0.033 (0.05) & 0.019 (0.03) & 0.005 (0.01) & 0.011 (0.02) & 0.077 (0.11) \\
 & $\pi_0$ & 0.101 & no & 0.199 (0.25) & 0.155 (0.19) & 0.084 (0.10) & 0.109 (0.14) & 0.311 (0.39) \\
 & 0.05 & 0.050 & yes & 0.001 (0.00) & 0.000 (0.00) & 0.000 (0.00) & 0.000 (0.00) & 0.001 (0.00) \\
 & 0.10 & 0.100 & no & 0.064 (0.10) & 0.041 (0.06) & 0.008 (0.01) & 0.021 (0.03) & 0.123 (0.19) \\
 & 0.20 & 0.200 & no & 0.966 (0.99) & 0.942 (0.98) & 0.902 (0.92) & 0.928 (0.96) & 0.974 (0.99) \\
\addlinespace
Qwen3-8B & $0.5\pi_0$ & 0.047 & yes & 0.000 (0.00) & 0.000 (0.00) & 0.000 (0.00) & 0.000 (0.00) & 0.000 (0.00) \\
 & $0.75\pi_0$ & 0.071 & yes & 0.002 (0.00) & 0.001 (0.00) & 0.001 (0.00) & 0.001 (0.00) & 0.004 (0.01) \\
 & $\pi_0$ & 0.095 & no & 0.016 (0.02) & 0.014 (0.02) & 0.009 (0.01) & 0.004 (0.00) & 0.039 (0.05) \\
 & 0.05 & 0.050 & yes & 0.000 (0.00) & 0.000 (0.00) & 0.000 (0.00) & 0.000 (0.00) & 0.000 (0.00) \\
 & 0.10 & 0.100 & no & 0.007 (0.01) & 0.007 (0.01) & 0.004 (0.00) & 0.004 (0.00) & 0.025 (0.03) \\
 & 0.20 & 0.200 & no & 0.577 (0.65) & 0.488 (0.52) & 0.399 (0.43) & 0.385 (0.41) & 0.731 (0.80) \\
\bottomrule
\end{tabular}
\end{table}

\begin{table}[H]
\centering\small
\caption{Left: test AUC of the frozen scorer, and coverage of rule~\eqref{eq:cert} (``excl.'', boundary trace excluded) and of the boundary-included variant (``incl.''), $\alpha=\pi_0$. Right: BH conformal selection at level $\alpha=\pi_0$: mean number of calibration errors $n_0$, the lattice minimum $\lceil m/(\alpha(n_0+1))\rceil$ of Proposition~\ref{prop:lattice} relative to the test size $m$, and the fraction of replicates with no rejection.}
\label{tab:bh}
\setlength{\tabcolsep}{2.6pt}
\begin{tabular}{l ccc ccc ccc rr ccc}
\toprule
 & \multicolumn{3}{c}{Confidence} & \multicolumn{3}{c}{J-lens} & \multicolumn{3}{c}{Probe} & & & \multicolumn{3}{c}{BH: no rejection} \\
\cmidrule(lr){2-4}\cmidrule(lr){5-7}\cmidrule(lr){8-10}\cmidrule(lr){13-15}
Model & AUC & excl. & incl. & AUC & excl. & incl. & AUC & excl. & incl. & $n_0$ & min.\ rej. & Conf. & J-lens & Probe \\
\midrule
Qwen3.5-0.8B & 0.70 & 0.218 & 0.213 & 0.66 & 0.080 & 0.077 & 0.81 & 0.489 & 0.486 & 107.8 & 2/119 & 0.09 & 0.21 & 0.01 \\
Qwen3.5-4B & 0.73 & 0.389 & 0.390 & 0.65 & 0.176 & 0.172 & 0.82 & 0.721 & 0.715 & 23.4 & 31/116 & 0.87 & 0.97 & 0.57 \\
OLMo-3-7B & 0.75 & 0.326 & 0.313 & 0.63 & 0.081 & 0.081 & 0.76 & 0.338 & 0.323 & 23.4 & 6/50 & 0.41 & 0.69 & 0.28 \\
Qwen3-1.7B & 0.66 & 0.094 & 0.088 & 0.61 & 0.042 & 0.047 & 0.76 & 0.212 & 0.216 & 17.9 & 15/58 & 0.96 & 0.96 & 0.76 \\
Qwen2.5-7B-It & 0.60 & 0.149 & 0.160 & 0.62 & 0.156 & 0.157 & 0.68 & 0.232 & 0.232 & 12.9 & 18/54 & 0.98 & 0.90 & 0.91 \\
Qwen3-4B & 0.66 & 0.199 & 0.189 & 0.61 & 0.155 & 0.148 & 0.71 & 0.311 & 0.300 & 14.3 & 79/114 & 1.00 & 1.00 & 1.00 \\
Qwen3-8B & 0.61 & 0.016 & 0.019 & 0.47 & 0.014 & 0.012 & 0.67 & 0.039 & 0.043 & 6.5 & 79/55 & 1.00 & 1.00 & 1.00 \\
\bottomrule
\end{tabular}
\end{table}

\begin{table}[H]
\centering\small
\caption{Gaussian score model with known selective risk; 20{,}000 calibration draws per row, $\delta=0.10$. $P(R>\alpha)$ is the quantity bounded by $(\alpha,\delta)$-validity; $P(R>\alpha\mid\text{fire})$ is the probability that an issued certificate is wrong; the next two columns give the median and 90th percentile of $R(\hat\lambda)/\alpha$ among failing draws, i.e.\ how far a wrong certificate overshoots its target.}
\label{tab:synthetic}
\setlength{\tabcolsep}{4pt}
\begin{tabular}{cccc l cccccc}
\toprule
$\pi_0$ & AUC & $n_{\mathrm{cal}}$ & $\alpha$ & Procedure & $P_{\mathrm{fire}}$ & $P(R>\alpha)$ & $P(R>\alpha\mid\text{fire})$ & med.\ $R/\alpha$ & p90 $R/\alpha$ & Mean cov. \\
\midrule
0.16 & 0.82 & 154 & $0.75\pi_0$ & CP, Eq.~(1) & 0.558 & 0.0014 & 0.003 & 1.06 & 1.13 & 0.322 \\
0.16 & 0.82 & 154 & $0.75\pi_0$ & Data-dep.\ walk & 0.985 & 0.0319 & 0.032 & 1.06 & 1.24 & 0.631 \\
\addlinespace
0.16 & 0.82 & 154 & $0.5\pi_0$ & CP, Eq.~(1) & 0.089 & 0.0019 & 0.021 & 1.04 & 1.15 & 0.049 \\
0.16 & 0.82 & 154 & $0.5\pi_0$ & Data-dep.\ walk & 0.770 & 0.0368 & 0.048 & 1.11 & 1.34 & 0.331 \\
\addlinespace
0.16 & 0.71 & 154 & $0.75\pi_0$ & CP, Eq.~(1) & 0.112 & 0.0018 & 0.016 & 1.04 & 1.14 & 0.054 \\
0.16 & 0.71 & 154 & $0.75\pi_0$ & Data-dep.\ walk & 0.724 & 0.0374 & 0.052 & 1.09 & 1.20 & 0.306 \\
\addlinespace
0.16 & 0.71 & 154 & $0.5\pi_0$ & CP, Eq.~(1) & 0.005 & 0.0026 & 0.520 & 1.12 & 1.28 & 0.003 \\
0.16 & 0.71 & 154 & $0.5\pi_0$ & Data-dep.\ walk & 0.276 & 0.0433 & 0.157 & 1.11 & 1.35 & 0.092 \\
\addlinespace
0.22 & 0.76 & 78 & $0.75\pi_0$ & CP, Eq.~(1) & 0.118 & 0.0020 & 0.017 & 1.07 & 1.22 & 0.061 \\
0.22 & 0.76 & 78 & $0.75\pi_0$ & Data-dep.\ walk & 0.712 & 0.0411 & 0.058 & 1.10 & 1.27 & 0.339 \\
\addlinespace
0.22 & 0.76 & 78 & $0.5\pi_0$ & CP, Eq.~(1) & 0.004 & 0.0025 & 0.694 & 1.11 & 1.23 & 0.002 \\
0.22 & 0.76 & 78 & $0.5\pi_0$ & Data-dep.\ walk & 0.309 & 0.0418 & 0.135 & 1.14 & 1.40 & 0.135 \\
\addlinespace
0.09 & 0.66 & 153 & $0.75\pi_0$ & CP, Eq.~(1) & 0.019 & 0.0027 & 0.148 & 1.12 & 1.24 & 0.011 \\
0.09 & 0.66 & 153 & $0.75\pi_0$ & Data-dep.\ walk & 0.337 & 0.0359 & 0.106 & 1.11 & 1.27 & 0.158 \\
\addlinespace
0.09 & 0.66 & 153 & $0.5\pi_0$ & CP, Eq.~(1) & 0.000 & 0.0001 & 1.000 & 1.59 & 1.62 & 0.000 \\
0.09 & 0.66 & 153 & $0.5\pi_0$ & Data-dep.\ walk & 0.089 & 0.0786 & 0.882 & 1.11 & 1.37 & 0.041 \\
\bottomrule
\end{tabular}
\end{table}

\begin{table}[H]
\centering\scriptsize
\caption{FS-floor on real data, mean coverage for each (model, signal) pair; GSM8K, $\delta=0.10$, 1{,}000 replicates. CP: rule~\eqref{eq:cert}. In-s.: FS-floor with the grid from in-sample training scores, $c=2$. $c{=}1,\dots,4$: cross-fitted grid at fixed $c$. Sim.: cross-fitted grid with $c$ chosen per cell by the pre-calibration simulation.}
\label{tab:fsfull}
\setlength{\tabcolsep}{2.2pt}
\begin{tabular}{ll cccccccc cccccccc}
\toprule
 & & \multicolumn{8}{c}{$\alpha=\pi_0$} & \multicolumn{8}{c}{$\alpha=0.75\pi_0$} \\
\cmidrule(lr){3-10}\cmidrule(lr){11-18}
Model & Signal & CP & In-s. & $c{=}1$ & $1.5$ & $2$ & $3$ & $4$ & Sim. & CP & In-s. & $c{=}1$ & $1.5$ & $2$ & $3$ & $4$ & Sim. \\
\midrule
Qwen3.5-0.8B & Conf. & 0.218 & 0.113 & 0.037 & 0.134 & 0.134 & 0.217 & 0.291 & 0.291 & 0.007 & 0.023 & 0.012 & 0.012 & 0.021 & 0.034 & 0.043 & 0.045 \\
 & J-lens & 0.080 & 0.076 & 0.007 & 0.030 & 0.030 & 0.094 & 0.134 & 0.134 & 0.001 & 0.011 & 0.002 & 0.002 & 0.007 & 0.011 & 0.013 & 0.013 \\
 & Probe & 0.489 & 0.278 & 0.319 & 0.502 & 0.502 & 0.616 & 0.652 & 0.652 & 0.063 & 0.064 & 0.108 & 0.108 & 0.151 & 0.180 & 0.219 & 0.219 \\
\addlinespace
Qwen3.5-4B & Conf. & 0.389 & 0.466 & 0.241 & 0.367 & 0.464 & 0.650 & 0.701 & 0.701 & 0.113 & 0.280 & 0.131 & 0.194 & 0.253 & 0.366 & 0.313 & 0.365 \\
 & J-lens & 0.176 & 0.223 & 0.018 & 0.100 & 0.166 & 0.269 & 0.336 & 0.336 & 0.024 & 0.074 & 0.014 & 0.045 & 0.067 & 0.091 & 0.097 & 0.098 \\
 & Probe & 0.721 & 0.736 & 0.481 & 0.719 & 0.822 & 0.893 & 0.905 & 0.905 & 0.390 & 0.544 & 0.362 & 0.517 & 0.644 & 0.699 & 0.705 & 0.705 \\
\addlinespace
OLMo-3-7B & Conf. & 0.326 & 0.424 & 0.056 & 0.251 & 0.378 & 0.543 & 0.573 & 0.573 & 0.099 & 0.220 & 0.046 & 0.122 & 0.218 & 0.268 & 0.259 & 0.266 \\
 & J-lens & 0.081 & 0.198 & 0.021 & 0.058 & 0.102 & 0.177 & 0.242 & 0.242 & 0.007 & 0.048 & 0.007 & 0.018 & 0.022 & 0.058 & 0.072 & 0.048 \\
 & Probe & 0.338 & 0.116 & 0.307 & 0.467 & 0.497 & 0.538 & 0.518 & 0.518 & 0.087 & 0.050 & 0.210 & 0.237 & 0.256 & 0.246 & 0.264 & 0.246 \\
\addlinespace
Qwen3-1.7B & Conf. & 0.094 & 0.071 & 0.043 & 0.057 & 0.078 & 0.167 & 0.193 & 0.193 & 0.008 & 0.020 & 0.010 & 0.013 & 0.024 & 0.060 & 0.058 & 0.059 \\
 & J-lens & 0.042 & 0.090 & 0.044 & 0.058 & 0.070 & 0.093 & 0.115 & 0.115 & 0.003 & 0.023 & 0.013 & 0.015 & 0.020 & 0.018 & 0.019 & 0.019 \\
 & Probe & 0.212 & 0.140 & 0.169 & 0.195 & 0.262 & 0.411 & 0.454 & 0.489 & 0.055 & 0.050 & 0.055 & 0.107 & 0.130 & 0.174 & 0.170 & 0.216 \\
\addlinespace
Qwen2.5-7B-It & Conf. & 0.149 & 0.219 & 0.039 & 0.058 & 0.171 & 0.236 & 0.352 & 0.352 & 0.012 & 0.071 & 0.008 & 0.033 & 0.049 & 0.124 & 0.088 & 0.086 \\
 & J-lens & 0.156 & 0.349 & 0.184 & 0.260 & 0.329 & 0.405 & 0.403 & 0.405 & 0.030 & 0.140 & 0.073 & 0.115 & 0.142 & 0.125 & 0.067 & 0.112 \\
 & Probe & 0.232 & 0.104 & 0.206 & 0.295 & 0.363 & 0.493 & 0.522 & 0.533 & 0.033 & 0.049 & 0.076 & 0.104 & 0.180 & 0.231 & 0.155 & 0.234 \\
\addlinespace
Qwen3-4B & Conf. & 0.199 & 0.283 & 0.025 & 0.143 & 0.274 & 0.375 & 0.391 & 0.396 & 0.033 & 0.132 & 0.010 & 0.079 & 0.121 & 0.168 & 0.181 & 0.159 \\
 & J-lens & 0.155 & 0.240 & 0.086 & 0.172 & 0.200 & 0.244 & 0.311 & 0.314 & 0.019 & 0.077 & 0.048 & 0.054 & 0.071 & 0.103 & 0.124 & 0.098 \\
 & Probe & 0.311 & 0.403 & 0.241 & 0.340 & 0.457 & 0.482 & 0.508 & 0.513 & 0.077 & 0.197 & 0.107 & 0.189 & 0.237 & 0.285 & 0.285 & 0.285 \\
\addlinespace
Qwen3-8B & Conf. & 0.016 & 0.108 & 0.027 & 0.109 & 0.116 & 0.155 & 0.000 & 0.150 & 0.002 & 0.047 & 0.023 & 0.030 & 0.039 & 0.000 & 0.000 & 0.033 \\
 & J-lens & 0.014 & 0.106 & 0.015 & 0.056 & 0.090 & 0.089 & 0.000 & 0.104 & 0.001 & 0.012 & 0.017 & 0.033 & 0.012 & 0.000 & 0.000 & 0.012 \\
 & Probe & 0.039 & 0.139 & 0.146 & 0.180 & 0.238 & 0.269 & 0.000 & 0.276 & 0.004 & 0.068 & 0.076 & 0.075 & 0.101 & 0.000 & 0.000 & 0.106 \\
\bottomrule
\end{tabular}
\end{table}

\begin{table}[H]
\centering\scriptsize
\caption{FS-floor against two cheap baselines on the real cells (means over 7 models $\times$ 3 signals, 1{,}000 replicates; each entry is mean coverage / $P_{\mathrm{fire}}$). $G$: number of Bonferroni grid points in rule~\eqref{eq:cert} ($G=20$ in the paper). Floor start: the single fixed threshold at $j_0$ tested at level $\delta$, i.e.\ FS-floor without the walk. Wins: pairs where FS-floor covers more than the baseline. The last two columns compare the $P_{\mathrm{fire}}$ predicted before calibration by the simulation that chooses $c$ with the realized $P_{\mathrm{fire}}$ across the 21 pairs (Spearman correlation and mean absolute error).}
\label{tab:baselines}
\setlength{\tabcolsep}{3.5pt}
\begin{tabular}{l ccccc cc cc}
\toprule
 & \multicolumn{3}{c}{Rule~\eqref{eq:cert}} & & & \multicolumn{2}{c}{FS-floor wins vs.} & \multicolumn{2}{c}{Predicted $P_{\mathrm{fire}}$} \\
\cmidrule(lr){2-4}\cmidrule(lr){7-8}\cmidrule(lr){9-10}
Target & $G{=}20$ & $G{=}10$ & $G{=}5$ & Floor start & FS-floor & $G{=}5$ & Floor start & Spearman & MAE \\
\midrule
$\pi_0$ & 0.211 / 0.33 & 0.242 / 0.37 & 0.259 / 0.40 & 0.186 / 0.50 & 0.390 / 0.50 & 21/21 & 21/21 & 0.72 & 0.154 \\
$0.75\pi_0$ & 0.051 / 0.10 & 0.067 / 0.13 & 0.072 / 0.14 & 0.101 / 0.28 & 0.163 / 0.28 & 21/21 & 21/21 & 0.70 & 0.104 \\
$0.5\pi_0$ & 0.004 / 0.01 & 0.007 / 0.01 & 0.010 / 0.02 & 0.024 / 0.07 & 0.033 / 0.07 & 20/21 & 21/21 & 0.72 & 0.046 \\
\bottomrule
\end{tabular}
\end{table}

\begin{table}[H]
\centering\scriptsize
\caption{Every shift cell (Section~\ref{sec:shift}): GSM8K-trained scorer and GSM8K-calibrated threshold ($\alpha$: training-partition base error on GSM8K) applied to the same model's traces on another benchmark. $\pi_{\mathrm{new}}$: base error rate on the new benchmark. Per signal: mean coverage and $\hat R$, the error among accepted traces pooled over firing replicates; bold where $\hat R>\alpha$. 1{,}000 replicates per cell.}
\label{tab:shift}
\setlength{\tabcolsep}{4pt}
\begin{tabular}{ll cc cc cc cc}
\toprule
 & & & & \multicolumn{2}{c}{Confidence} & \multicolumn{2}{c}{J-lens} & \multicolumn{2}{c}{Probe} \\
\cmidrule(lr){5-6}\cmidrule(lr){7-8}\cmidrule(lr){9-10}
Model & Benchmark & $\pi_{\mathrm{new}}$ & $\alpha$ & Cov. & $\hat R$ & Cov. & $\hat R$ & Cov. & $\hat R$ \\
\midrule
Qwen3.5-0.8B & ARC-C & 0.77 & 0.64 & 0.00 & 0.61 & 0.06 & \textbf{0.83} & 0.27 & \textbf{0.75} \\
 & BBH-logic & 0.61 & 0.64 & 0.07 & 0.55 & 0.05 & 0.59 & 0.09 & 0.60 \\
 & CSQA & 0.79 & 0.64 & 0.00 & 0.01 & 0.05 & \textbf{0.86} & 0.27 & \textbf{0.76} \\
 & StrategyQA & 0.47 & 0.64 & 0.00 & 0.49 & 0.10 & 0.47 & 0.13 & 0.44 \\
\addlinespace
Qwen3.5-4B & ARC-C & 0.07 & 0.16 & 0.11 & 0.02 & 0.27 & 0.07 & 0.60 & 0.06 \\
 & BBH-logic & 0.19 & 0.16 & 0.29 & \textbf{0.20} & 0.28 & \textbf{0.19} & 0.60 & \textbf{0.17} \\
 & CSQA & 0.20 & 0.16 & 0.05 & 0.10 & 0.23 & 0.16 & 0.52 & 0.15 \\
 & StrategyQA & 0.33 & 0.16 & 0.09 & \textbf{0.27} & 0.23 & \textbf{0.34} & 0.59 & \textbf{0.29} \\
\addlinespace
OLMo-3-7B & ARC-C & 0.26 & 0.34 & 0.21 & 0.22 & 0.09 & 0.32 & 0.44 & 0.24 \\
 & BBH-logic & 0.63 & 0.34 & 0.44 & \textbf{0.60} & 0.07 & \textbf{0.57} & 0.35 & \textbf{0.62} \\
 & CSQA & 0.30 & 0.34 & 0.20 & 0.33 & 0.06 & 0.32 & 0.52 & 0.33 \\
\addlinespace
Qwen3-1.7B & ARC-C & 0.19 & 0.21 & 0.09 & 0.12 & 0.14 & 0.18 & 0.11 & 0.13 \\
 & CSQA & 0.35 & 0.21 & 0.10 & \textbf{0.29} & 0.14 & \textbf{0.35} & 0.10 & \textbf{0.29} \\
 & StrategyQA & 0.36 & 0.21 & 0.04 & \textbf{0.32} & 0.14 & \textbf{0.36} & 0.06 & \textbf{0.36} \\
\addlinespace
Qwen2.5-7B-It & ARC-C & 0.11 & 0.22 & 0.34 & 0.12 & 0.41 & 0.12 & 0.35 & 0.11 \\
 & BBH-logic & 0.41 & 0.22 & 0.33 & \textbf{0.42} & 0.42 & \textbf{0.45} & 0.39 & \textbf{0.40} \\
 & CSQA & 0.19 & 0.22 & 0.35 & 0.17 & 0.43 & 0.16 & 0.41 & 0.17 \\
 & StrategyQA & 0.27 & 0.22 & 0.34 & \textbf{0.26} & 0.37 & \textbf{0.25} & 0.42 & \textbf{0.26} \\
\addlinespace
Qwen3-4B & ARC-C & 0.15 & 0.10 & 0.31 & \textbf{0.11} & 0.29 & \textbf{0.14} & 0.37 & \textbf{0.12} \\
 & BBH-logic & 0.18 & 0.10 & 0.36 & \textbf{0.18} & 0.27 & \textbf{0.19} & 0.42 & \textbf{0.18} \\
 & StrategyQA & 0.26 & 0.10 & 0.27 & \textbf{0.23} & 0.32 & \textbf{0.24} & 0.30 & \textbf{0.24} \\
\addlinespace
Qwen3-8B & ARC-C & 0.08 & 0.09 & 0.02 & 0.02 & 0.01 & 0.07 & 0.10 & 0.08 \\
 & CSQA & 0.18 & 0.09 & 0.01 & \textbf{0.12} & 0.01 & \textbf{0.18} & 0.08 & \textbf{0.17} \\
 & StrategyQA & 0.23 & 0.09 & 0.02 & \textbf{0.23} & 0.01 & \textbf{0.22} & 0.13 & \textbf{0.23} \\
\bottomrule
\end{tabular}
\end{table}

\begin{table}[H]
\centering\scriptsize
\caption{Best-of-$n$ selection (Section~\ref{sec:pressure}) per model and signal. $P_{\mathrm{fire}}$: firing probability of the naive arm (independent of $n$). $\hat R/\alpha$: error among accepted BoN-selected test traces relative to the target, naive arm (-- where fewer than 2\% of replicates fire). Corrected: mean coverage of the selection-conditional arm at $n=1$ and $n=32$. $\Delta$acc: change in accuracy of the selected trace from $n=1$ to $n=32$, and from its peak over $n$ to $n=32$.}
\label{tab:bon}
\setlength{\tabcolsep}{4pt}
\begin{tabular}{ll c ccc cc cc}
\toprule
 & & & \multicolumn{3}{c}{$\hat R/\alpha$, naive} & \multicolumn{2}{c}{Corrected cov.} & \multicolumn{2}{c}{$\Delta$acc} \\
\cmidrule(lr){4-6}\cmidrule(lr){7-8}\cmidrule(lr){9-10}
Model & Signal & $P_{\mathrm{fire}}$ & $n{=}1$ & $n{=}8$ & $n{=}32$ & $n{=}1$ & $n{=}32$ & $n{=}32$ & from peak \\
\midrule
Qwen3.5-4B & confidence & 0.031 & 0.92 & 1.14 & 1.25 & 0.019 & 0.002 & -0.004 & -0.022 \\
 & J-lens & 0.015 & -- & -- & -- & 0.008 & 0.000 & +0.001 & -0.016 \\
 & probe & 0.136 & 0.70 & 0.89 & 0.69 & 0.087 & 0.222 & +0.038 & +0.000 \\
 & logit lens & 0.014 & -- & -- & -- & 0.008 & 0.000 & -0.017 & -0.032 \\
 & surface & 0.002 & -- & -- & -- & 0.001 & 0.000 & -0.040 & -0.040 \\
\addlinespace
Qwen3-4B & confidence & 0.093 & 0.78 & 0.91 & 1.12 & 0.054 & 0.005 & +0.001 & -0.006 \\
 & J-lens & 0.007 & -- & -- & -- & 0.006 & 0.000 & +0.005 & -0.005 \\
 & probe & 0.045 & 0.87 & 1.33 & 1.52 & 0.026 & 0.000 & -0.029 & -0.031 \\
 & logit lens & 0.008 & -- & -- & -- & 0.004 & 0.000 & -0.001 & -0.003 \\
 & surface & 0.007 & -- & -- & -- & 0.005 & 0.000 & -0.017 & -0.017 \\
\bottomrule
\end{tabular}
\end{table}

\begin{table}[H]
\centering\small
\caption{GSM8K out-of-fold AUC (5-fold, grouped by problem, logistic regression) of each signal, and paired bootstrap differences ($^{*}$: two-sided $p<0.05$, uncorrected). ``Random'': J-lens features with a norm-matched random transport. These are cross-validated AUCs on all traces and differ from the frozen-scorer test AUCs of Table~\ref{tab:bh}.}
\label{tab:auc}
\setlength{\tabcolsep}{4pt}
\begin{tabular}{lrr cccccc ccc}
\toprule
Model & Traces & Acc. & Conf. & J-lens & Logit lens & Surface & Random & Probe & J$-$Conf. & J$-$Surf. & Probe$-$J \\
\midrule
Qwen3.5-0.8B & 1244 & 0.33 & 0.714 & 0.713 & 0.661 & 0.531 & 0.641 & 0.842 & -0.001 & +0.182$^{*}$ & +0.129$^{*}$ \\
Qwen3.5-4B & 1365 & 0.84 & 0.715 & 0.707 & 0.669 & 0.515 & 0.553 & 0.847 & -0.009 & +0.191$^{*}$ & +0.140$^{*}$ \\
OLMo-3-7B & 295 & 0.66 & 0.748 & 0.642 & 0.595 & 0.613 & 0.594 & 0.770 & -0.106$^{*}$ & +0.029 & +0.128$^{*}$ \\
Qwen3-1.7B & 373 & 0.78 & 0.675 & 0.679 & 0.635 & 0.612 & 0.604 & 0.810 & +0.005 & +0.067 & +0.131$^{*}$ \\
Qwen2.5-7B-It & 323 & 0.81 & 0.584 & 0.648 & 0.522 & 0.602 & 0.559 & 0.668 & +0.065 & +0.046 & +0.020 \\
Qwen3-4B & 1387 & 0.91 & 0.610 & 0.607 & 0.551 & 0.549 & 0.625 & 0.751 & -0.003 & +0.057 & +0.144$^{*}$ \\
Qwen3-8B & 350 & 0.91 & 0.592 & 0.344 & 0.411 & 0.271 & 0.479 & 0.809 & -0.248$^{*}$ & +0.073 & +0.466$^{*}$ \\
\bottomrule
\end{tabular}
\end{table}

\end{document}